\documentclass[a4paper,10pt]{article}
\usepackage{mathptmx}
\usepackage[latin1]{inputenc}
\usepackage[T1]{fontenc}
\usepackage[cyr]{aeguill}
\usepackage{amsmath,amssymb,amsthm,graphicx,amsfonts}
\usepackage{lmodern}
\usepackage[french,english]{babel}

\newtheorem{prop}{Proposition}[section]

\newtheorem{theo}[prop]{Theorem}
\newtheorem{lem}[prop]{Lemma}
\theoremstyle{definition}
\newtheorem{defin}[prop]{Definition}

\newcommand{\I}{\mathbb{I}}

\renewcommand{\v}{\>v}

\newcommand{\x}{\>x}

\newcommand{\F}{\mathbb{F}}
\newcommand{\V}{\mathbb{V}}

\newcommand{\W}{\mathcal{W}}
\newcommand{\N}{\mathcal{N}}
\newcommand{\R}{\mathbb{R}}
\newcommand{\E}{\mathbb{E}}
\newcommand{\G}{\mathcal{G}}

\newcommand{\n}{^{\scriptscriptstyle (n)}}

\newcommand{\1}{\leavevmode\hbox{\rm \small1\kern-0.35em\normalsize1}}
\newcommand{\ind}[1]{\1_{\{#1\}}}

\def\egaldef{\stackrel{\mbox{\tiny def}}{=}}

\title{Local stability of Belief Propagation algorithm with multiple fixed
points}
\author{Victorin Martin\\INRIA Paris-Rocquencourt
  \and Jean-Marc Lasgouttes\\INRIA Paris-Rocquencourt
  \and Cyril Furtlehner\\INRIA Saclay
}
\date{May 25, 2012}
\begin{document}

\maketitle

\begin{abstract}
$\ $A number of problems in statistical physics and
computer science can be expressed as the computation of marginal
probabilities over a Markov random field. Belief propagation, an iterative
message-passing algorithm, computes exactly such marginals when the underlying graph
is a tree. But it has gained its popularity as an
efficient way to approximate them in the more general case, even if it can exhibits
multiple fixed points and is not guaranteed to converge. In this
paper, we express a new sufficient condition for local stability of a belief
propagation fixed point in terms of the graph structure and the beliefs values at the
fixed point. This gives credence to the usual understanding that Belief Propagation
performs better on sparse graphs.
\end{abstract}
\vspace{5mm}
\noindent\begin{scriptsize}{Submitted to: \textit{Starting Artificial Intelligence
Research Symposium 2012}}\end{scriptsize}

\section{Introduction}

We consider in this work a Markov random field (MRF) on a finite
graph with local interactions, on which we want to compute marginal
probabilities. The structure of the underlying model is described by a
set of discrete variables $\x=\{x_i,i\in \V\}\in\{1,\ldots,q\}^\V$,
where the set $\V$ of variables is linked together by so-called
``factors'' which are subsets $a\subset \V$ of variables. If $\F$ is
this set of factors, we consider the set of probability measures of
the form
\begin{equation}\label{eq:joint}
p(\x) = \prod_{i\in\V}\phi_i(x_i) \prod_{a\in\F}\psi_a(\x_a),
\end{equation}
where $\x_a=\{x_i, i\in a\}$. In what follows, a factor will be
indifferently considered as a node in a graph or as a set of variables.
In this respect, $i\in a$ ca be read as ``the variable node $i$ is
connected to the factor node $a$.''

$\F$ together with $\V$ define the factor graph $\G$ \cite{Kschi},
which is an undirected bipartite graph. We will also assume that $p$ is strictly
positive, which is to say that the MRF exhibits no deterministic behavior. The set
$\E$ of edges contains all the pairs $(a,i)\in\F\times\V$ such that $i\in a$. We
denote $d_a$ (resp.\ $d_i$) the degree of the factor node $a$ (resp.\ of the variable
node $i$).

Exact procedures for computing marginal probabilities of $p$ generally
face an exponential complexity and one has to resort to approximate
procedures. 
In computer science, the belief propagation (BP) algorithm~\cite{Pearl} is a
message passing procedure that allows to compute efficiently exact
marginal probabilities when the underlying graph is a tree. When the
graph has cycles, it is still possible to apply the procedure, which
converges with a rather good accuracy on sufficiently sparse graphs.
However, there may be several fixed points, corresponding to stationary points of
the Bethe free energy~\cite{YeFrWe3}. Stable fixed points of BP are local minima of
the Bethe free energy~\cite{Heskes4,WaFu}.

The question of convergence of BP has been addressed in a series of works
\cite{Tatikonda02,Ihler,MooijKappen07}, which establish sufficient conditions on the
MRF under which BP converges to a unique fixed point. However, cases with multiple
fixed points can be used to encode different patterns~\cite{FuLaAu} and have not
been studied yet. Wainwright~\cite{Wain3} suggests that, facing the joint problem of
parameter estimation and prediction in a MRF, estimation under the Bethe
approximation and prediction using BP is an efficient setting. This consist in
choosing \eqref{eq:joint} such that one fixed point is known. We propose here to
change the viewpoint and, instead of looking for conditions ensuring a single fixed
point, examine the local properties of each of them. Theorem \ref{thm:stability}
gives a sufficient condition for local stability of fixed points which quantifies
the known fact that BP performs better in sparser graphs.

The paper is organized as follows: the BP algorithm and its various
normalization strategies are defined in Section~\ref{sec:algorithm}.
Section~\ref{sec:bdynamic} exhibits cases where convergence of
messages is equivalent to convergence of beliefs, allowing us to consider only
message convergence. Finally in Section~\ref{sec:stability}, we provide some
sufficient conditions for local stability of BP fixed points.
Section~\ref{sec:conclusion} concludes the paper.

\section{The belief propagation algorithm}\label{sec:algorithm}

The belief propagation algorithm~\cite{Pearl} is a message passing
procedure, whose output is a set of estimated marginal probabilities,
the beliefs $b_a(\x_a)$ (including single nodes beliefs $b_i(x_i)$).
The idea is to factor the marginal probability at a given site as a
product of contributions coming from neighboring factor nodes, which
are the messages. With definition \eqref{eq:joint} of the joint
probability measure, the updates rules read:
\begin{align}
m_{a\to i}(x_i) &\gets
\sum_{\x_{a\setminus i}} \psi_a(\x_a)\prod_{j\in a\setminus i} 
n_{j\to a }(x_j), \label{urules}\\[0.2cm]
n_{i \to a}(x_i) &\egaldef \phi_i(x_i)\prod_{a'\ni i, a'\ne a}
m_{a'\to i}(x_i), \label{urulesn}
\end{align}
where the notation $\sum_{\x_{a\setminus i}}$ should be understood as
summing from $1$ to $q$ all the variables $x_j$, $j\in a\subset \V$,
$j\ne i$. At any point of the algorithm, one can
compute the current beliefs as
\begin{align}
b_i(x_i) &\egaldef 
\frac{1}{Z_i(m)}\phi_i(x_i)\prod_{a\ni i} m_{a\to i}(x_i),
\label{belief1}\\[0.2cm]
b_a(\x_a) &\egaldef 
\frac{1}{Z_a(m)}\psi_a(\x_a)\prod_{i\in a} n_{i\to a}(x_i),
\label{belief2} 
\end{align}
where $Z_i(m)$ and $Z_a(m)$ are the normalization constants that
ensure that
\begin{equation}
\label{eq:normb}
 \sum_{x_i} b_i(x_i) = 1,\qquad\sum_{\x_a}b_a(\x_a)=1.
\end{equation}
These constants reduce to $1$ when $\G$ is a tree. When the algorithm has converged,
the following compatibility condition holds :
\begin{equation}
\label{eq:compat}
\sum_{\x_{a \setminus i}} b_a(\x_a) = b_i(x_i).
\end{equation}
In practice, the messages are often normalized so that
\begin{equation}\label{eq:normalization}
 \sum_{x_i=1}^q m_{a\to i}(x_i)= 1.
\end{equation}
However, the possibilities of normalization are not limited to this
setting. Consider the mapping
\begin{equation}\label{eq:theta}
 \Theta_{ai,x_i}(m)\egaldef
\sum_{\x_{a\setminus i}} \psi_a(\x_a)
\prod_{j\in a\setminus i}\biggl[\phi_j(x_j)
\prod_{a'\ni j, a'\ne a}m_{a'\to j}(x_j)\biggr].
\end{equation}
A normalized version of BP is defined by the update rule
\begin{equation}\label{eq:normrule}
  \tilde m_{a\to i}(x_i)
\gets\frac{\Theta_{ai,x_i}(\tilde m)}
       {Z_{ai}(\tilde m)}.
\end{equation}
where $Z_{ai}(\tilde m)$ is a constant that depends on the messages and
which, in the case of~(\ref{eq:normalization}), reads
\begin{equation}\label{eq:Zmess}
Z^\mathrm{mess}_{ai}(\tilde m) \egaldef \sum_{x=1}^q\Theta_{ai,x}(\tilde
m).
\end{equation}
Following~\cite{Wain}, it is worth noting that (\ref{urules},\ref{urulesn}) can be
rewritten as
\begin{equation}\label{urule2}
 m_{a\to i}(x_i)\gets \frac{Z_a(m)b_{i|a}(x_i)}{Z_i(m)b_i(x_i)}m_{a\to
i}(x_i),
\end{equation}
where we use the convenient shorthand notation $ b_{i|a}(x_i) \egaldef
\sum_{\x_{a\setminus i}}b_a(\x_a)$. This suggests a different type of normalization,
used in particular by~\cite{Heskes4}, namely
\begin{equation}\label{eq:Zbel}
 Z^\text{bel}_{ai}(\tilde m) \egaldef \frac{Z_a(\tilde m)}{Z_i(\tilde m)},
\end{equation}
which leads to the simple update rule
\begin{equation}
\label{eq:simple_uprule}
 \tilde m_{a\to i}(x_i)\gets \frac{b_{i|a}(x_i)}{b_i(x_i)}\tilde
 m_{a\to i}(x_i).
\end{equation}

\section{Belief and message dynamic}
\label{sec:bdynamic}

At each step of the algorithm, using \eqref{belief1} and \eqref{belief2}, we can
compute the current beliefs $b\n_i$ and $b\n_a$ associated with the message $m\n$.
The sequence $m\n$ will be said to be ``$b$-convergent'' when the
sequences $b\n_i$ and $b\n_a$ converge. This is the convergence that
is interesting in practice. The term ``$m$-convergence''
will be used to refer to convergence of the sequence $m\n$ itself.
Since the algorithm is expressed in terms messages, $m$-convergence
obviously implies $b$-convergence, but the opposite is not generally true.
The aim of this section is
to provide a broad class of normalization policies such that $b$- and
$m$-convergence, are
equivalent in order to focus on $m$-convergence in the next section.

As pointed out in~\cite{MooijKappen07}, different sets of messages
correspond to the same set of beliefs. The following lemma makes this
explicit.

\begin{lem}
\label{lem:bel_inv}
Two set of messages $m$ and $m'$ lead to the same beliefs if, and only
if, there is a set of strictly positive constants $c_{ai}$ such that
\[
  m'_{a\to i}(x_i) = c_{ai} m_{a\to i}(x_i).
\]
\end{lem}
\begin{proof}
The direct part of the lemma is trivial. Concerning the other part, we
have from (\ref{belief1}) and (\ref{belief2})
\begin{gather*}
 \frac{b_a(\x_a)Z_a(m)}{\psi_a(\x_a)} = \prod_{j \in a}
\prod_{b \ni j, b \ne a} m_{b\to j}(x_j),\\
 \frac{b_i(x_i)Z_i(m)}{\phi_i(x_i)} = \prod_{a \ni i} m_{a\to i}(x_i).
\end{gather*}
Assume the two vectors of messages $m$ and $m'$ lead to the
same set of beliefs $b$ and write $m_{a\to i}(x_i) = c_{ai}(x_i)\,
m'_{a\to i}(x_i)$. Then, from the relation on $b_i(x_i)$, the vector $\>c$
satisfies
\begin{equation}
\label{eq:inv_bi}
\prod_{a \ni i} c_{ai}(x_i)=\prod_{a \ni i} \frac{m_{a\to i}(x_i)}{m'_{a\to i}(x_i)}=
\frac{Z_i(m)}{Z_i(m')}
\egaldef v_i.
\end{equation}
Moreover, we want to preserve the beliefs $b_a$. Using~\eqref{eq:inv_bi},
we have
\begin{equation}
\label{eq:va_def}
 \prod_{j \in a}
c_{aj}(x_j) = \prod_{j \in a} \frac{m_{a\to j}(x_j)}{m'_{a\to j}(x_j)} =
\frac{Z_a(m')}{Z_a(m)} \prod_{i \in a} v_i \egaldef v_a,
\end{equation}
since $v_i$ (resp. $v_a$) does not depend on the choice of $x_i$ (resp.
$\x_a$), \eqref{eq:va_def} implies the independence of $c_{ai}(x_i)$ with
respect to $x_i$. Indeed, if we compare two vectors $\x_a$ and $\x_a'$
such that,
for all $i \in a\setminus j$, $x'_i = x_i$, but $x'_j \ne x_j$, then
 $c_{aj}(x_j) = c_{aj}(x'_j)$, which concludes the proof.
\end{proof}
Following an idea developed in \cite{MooijKappen07}, it is natural to
look at the behavior of BP in a quotient space corresponding to the
invariance of beliefs. First, we will introduce a natural
parametrization for which the quotient space is just a vector space.
Then we will show that, in terms of $b$-convergence, the effect of
normalization is null. Let us consider the following change of
variables:
\begin{equation*}
 \mu_{a\to i}(x_i) \egaldef \log m_{a\to i}(x_i),
\end{equation*}
so that the plain update mapping (\ref{eq:theta}) becomes
\begin{equation*}
 \mu_{a\to i}(x_i) \leftarrow \Lambda_{ai,x_i}(\mu) \egaldef \log \left[\sum_{\x_a
\setminus i}
\psi_a(\x_a)\exp
\Bigl(\sum_{j \in a \setminus i} \sum_{\substack{b \ni j\\ b
\ne a}} \mu_{b\to j}(x_j)\Bigr)\right].
\end{equation*}
We have $\mu \in \N \egaldef \R^{|\E|\times q}$ and we define the vector
space $\W$ which is the linear span of the following vectors
$\{e_{ai} \in \N\}_{(ai)\in\E}$
\[
 (e_{ai})_{cj,x_j} \egaldef \1_{\{a=c,i=j\}}.
\]
The invariance set of the beliefs corresponding to $\mu$ is simply the affine space
$\mu + \W$ (Lemma~\ref{lem:bel_inv}). So $\mu^{(n)}$ is $b$-convergent iff
$\mu^{(n)}$ converges in the quotient space $\N \setminus \W$, which is simply a
vector space \cite{Halmos}. We use the notation $[x]$ for the
canonical projection of $x$ on $\N \setminus \W$.

The normalization of $\mu$ leads to $\mu + w$ with some $w \in \W$. Indeed we have
\begin{equation*}
 \Lambda_{ai,x_i}(\mu + w) = \log \Bigl(\sum_{j \in a \setminus i}
\sum_{\substack{b \ni j\\b \ne a}} w_{bj} \Bigr) + \Lambda_{ai,x_i}(\mu)
 \egaldef l_{ai} + \Lambda_{ai,x_i}(\mu),
\end{equation*}
which can be summed up by $[\Lambda(\mu + \W)] = [\Lambda(\mu)]$, since $l\in \W$.
This means that normalization plays no role in $\N\setminus\W$ and
implies the following proposition.
\begin{prop}
\label{prop:dynamic}
The dynamic, i.e.\ the value of the normalized beliefs at each step, of the
BP algorithm with or without normalization is exactly the same.
\end{prop}
We will come back to this vision in terms of quotient space in
Section~\ref{ssec:b2qs}, and we now exhibit a broad class of
normalizations for which $b$-convergence and $m$-convergence are
equivalent.
\begin{defin}
A normalization $Z_{ai}$ is said to be \emph{positive homogeneous}
when it is of the form $Z_{ai} = N_{ai} \circ \Theta_{ai}$, with
$N_{ai}\,:\;\R_+^q\mapsto \R_+$ a positive homogeneous function of order $1$
satisfying
\begin{align}
\label{eq:Nlin}
 N_{ai}(\lambda m_{a\to i}) &=
\lambda N_{ai}(m_{a\to i}), \forall \lambda \geq 0.\\
\label{eq:Nai_pos}
N_{ai}(m_{a\to i}) &= 0 \iff  m_{a \to i} = 0.
\end{align}
\end{defin}
A particular family of positive homogeneous normalizations is obtained
when $N_{ai}$ is a norm on $\R^q$. This is the case the
normalization $Z_{ai}^\mathrm{mess}$ \eqref{eq:Zmess}. It is actually
not necessary to have a proper norm: the scheme used in \cite{WaFu} 
amounts to $Z^1_{ai}(m) \egaldef \Theta_{ai,1}(m)$. 

Note however that $Z_{ai}^\text{bel}$ \ref{eq:Zbel} is not part
positive homogeneous, and therefore the results of this section do not
apply to this case.
\begin{prop}
 \label{prop:bconv_mconv}
For any positive homogeneous normalization $Z_{ai}$ with  \emph{continuous}
$N_{ai}$, $m$-convergence and $b$-convergence are equivalent.
\end{prop}

\begin{proof}
Assume that the sequences of beliefs are such that $b\n_a\to b_a$ and $b\n_i\to b_i$
as $n\to\infty$. The idea of the proof is to first express the normalized messages
$\tilde m\n_{a\to i}$ at each step in terms of these beliefs, and then to conclude
by a continuity argument. Starting from a rewrite of
\eqref{belief1}--\eqref{belief2},
\begin{align*}
b_i\n(x_i) &= \frac{\phi_i(x_i)}{Z_i(\tilde m\n)}\prod_{a \ni i} \tilde m_{a\to
i}\n(x_i), \\
 b_a\n(\x_a)  &=\frac{\psi_a(\x_a)}{Z_a(\tilde m\n)}\prod_{j \in a} \phi_j(x_j)
\prod_{b \ni j, b
\ne a}
\tilde m_{b\to j}\n(x_j),
\end{align*}
one obtains by recombination
\[
 \prod_{j \in a} \tilde m\n_{a\to j}(x_j) 
 = \frac{K\n_{ai}(\x_{a\setminus i};x_i)}{\tilde Z_{ai}(\tilde m)},
\]
where an arbitrary variable $i\in a$ has been singled out and
\[
 \frac{1}{\tilde Z_{ai}(\tilde m)}\egaldef\frac{\prod_{j \in a} Z_j(\tilde
m\n)}{Z_a(\tilde m\n)},\quad K\n_{ai}(\x_{a\setminus i};x_i)
\egaldef \psi_a(\x_a) \frac{\prod_{j \in a} b_j\n(x_j)}{b_a\n(\x_a)}. 
\]

Assume now that $\x_{a\setminus i}$ is fixed and consider
$\>K\n_{ai}(\x_{a\setminus i})\egaldef K\n_{ai}(\x_{a\setminus
  i};\cdot)$ as a vector of $\R^q$. Normalizing each side of the
equation with a positive homogeneous function $N_{ai}$ yields
\[
 \frac{\tilde m\n_{a\to i}(x_i)}{N_{ai}\bigl[\tilde m\n_{a\to
i}\bigr]}
 = \frac{K\n_{ai}(\x_{a\setminus
i};x_i)}{N_{ai}\bigl[\>K\n_{ai}(\x_{a\setminus i})\bigr]}.
\]
Actually $N_{ai}\bigl[\tilde m\n_{a\to i}\bigr]=1$, since
$\tilde m\n_{a\to i}$ has been normalized by $N_{ai}$ and therefore
\[
 \tilde m\n_{a\to i}(x_i)
 = \frac{K\n_{ai}(\x_{a\setminus
i};x_i)}{N_{ai}\bigl[\>K\n_{ai}(\x_{a\setminus i})\bigr]}.
\]
This concludes the proof, since $\tilde m\n_{a\to i}$ has been expressed
as a continuous function of $b\n_i$ and $b\n_a$, and therefore it
converges whenever the beliefs converge.
\end{proof}

\section{Local stability of BP fixed points}\label{sec:stability}

The question of convergence of BP has been addressed in a series of
works \cite{Tatikonda02,Ihler,MooijKappen07} which establish
conditions and bounds on the MRF coefficients for having global
convergence. In this section, we change the viewpoint and, instead of
looking for conditions ensuring a single fixed point, we examine the local
properties each fixed point.

In what follows, we are interested in the local stability of a message
fixed point $m$ with associated beliefs $b$. It is known that a BP
fixed point is locally attractive if the Jacobian of the relevant
mapping ($\Theta$ or its normalized version) at this point has all its eigenvalues of
modulus strictly smaller than $1$ and unstable when, at least, one eigenvalue has a
modulus strictly greater than $1$.
The characterization of the local stability relies on two ingredients.
The first one is the oriented line graph $L(\G)$ based on $\G$, whose
vertices are the elements of $\E$, and whose oriented links relate
$ai$ to $a'j$ if $j\in a\cap a'$, $j\ne i$ and $a'\ne a$. The
corresponding $0$-$1$ adjacency matrix $A$ is defined by the
coefficients
\begin{equation}\label{defA}
A_{ai}^{a'j} \egaldef\ind{j\in a\cap a',\,j\ne i,\, a'\ne a}.
\end{equation}

The second ingredient is the set of stochastic matrices $B^{(iaj)}$, attached to
pairs of variables $(i,j)$ having a factor node $a$ in common, and which
coefficients at row $k$, column $\ell$ (in $\{1,\ldots,q\}^2$) are the conditional
beliefs
\begin{equation*}
b_{k\ell}^{(iaj)} \egaldef b_a(x_j=\ell\vert x_i = k)
= \sum_{\x_{a\setminus\{i,j\}}}\frac{b_a(\x_a)}{b_i(x_i)}
\Bigg|_{\substack{x_i=k\\ x_j=\ell}}.
\end{equation*}

\subsection{The unnormalized algorithm}

Let us first consider briefly the unnormalized
algorithm~(\ref{urules},\ref{urulesn}). Using the
representation~(\ref{urule2}), the Jacobian reads
at this point:
\begin{align*}
\frac{\partial \Theta_{ai,x_i}(m)}{\partial m_{a'\to j}(x_j)}
&=\sum_{\x_{a\setminus \{i,j\}}} 
       \frac{b_a(\x_a)}{b_i(x_i)}
       \frac{m_{a\to i}(x_i)}{m_{a'\to j}(x_j)}
       \ind{j\in a\setminus i}\ind{a' \ni j, a'\ne a}\\
&= \frac{b_{ij|a}(x_i,x_j)}{b_i(x_i)}
   \frac{m_{a\to i}(x_i)}{m_{a'\to j}(x_j)}A_{ai}^{a'j}
\end{align*}

Therefore, the Jacobian of the plain BP algorithm is---using a trivial
change of variable---similar to the
matrix $J$ defined, for any pair $(ai,k)$ and $(a'j,\ell)$ of
$\E\times\{1,\ldots,q\}$ by the elements
\begin{equation*}
J_{ai,k}^{a'j,\ell}  \egaldef b_{k\ell}^{(iaj)} A_{ai}^{a'j}.
\end{equation*}
This expression is analogous to the Jacobian encountered in
\cite{MooijKappen07}. It is interesting to note that it only depends
on the structure of the graph and on the belief corresponding to the
fixed point. Since $\G$ is a singly connected graph, it is clear that $A$ is an
irreducible matrix. To simplify the discussion, we assume in the
following that $J$ is also irreducible. This will be true as long as
the $\psi$ are always positive.

It can be shown~\cite{MaLaFu} that the spectral radius of $J$ is
always larger than $1$, except in some special cases where the number
of cycles in the graph is less than $1$. We will not develop this point here.

\subsection{Positive homogeneous normalization}\label{ssec:stabnorm}

We have seen in Proposition~\ref{prop:bconv_mconv} that all the
continuous positively homogeneous normalizations make $m$-convergence equivalent to
$b$-convergence. Since they all share the same properties, we look at the particular
case of $Z_{ai}^\text{mess}(m)$, which is both simple and differentiable. The
coefficients of the Jacobian matrix at fixed point $m$ with beliefs $b$ read
\begin{equation*}
\frac{\partial}{\partial \tilde m_{a'\to j}(\ell)}
\biggl[\frac{\Theta_{ai,k}(\tilde m)}
       {\sum_{x=1}^q\Theta_{ai,x}(\tilde m)}\biggr]
= J_{ai,k}^{a'j,\ell}\frac{m_{a\to i}(k)}{m_{a'\to j}(\ell)}
   - m_{a\to i}(k)\sum_{x=1}^q J_{ai,x}^{a'j,\ell}
                               \frac{m_{a\to i}(x)}{m_{a'\to j}(\ell)},
\end{equation*}
which is similar to  the matrix $\widetilde J$ of general term
\begin{equation}\label{eq:jacobn}
\widetilde J_{ai,k}^{a'j,\ell}
\egaldef\biggl[b_{k\ell}^{(iaj)}- \sum_{x=1}^q m_{a\to i}(x) b_{x\ell}^{(iaj)}
\biggr] A_{ai}^{a'j}
= J_{ai,k}^{a'j,\ell} - \sum_{x=1}^q m_{a\to i}(x)J_{ai,x}^{a'j,\ell},
\end{equation}
which can be summarized by $\tilde J = (\I - M)J$, with $\I$ the identity matrix and
$M$:
\[
 M_{ai,k}^{a'j,\ell} \egaldef m_{a'\to j}(\ell)\1_{\{a=b,i=j\}}. 
\]
The presence of the messages in the Jacobian $\widetilde J$ seems to
complicate the study, but in fact the spectrum of $\tilde J$ does not
depend on the messages themselves. It is known (see e.g.~\cite{FuLaAu}) that it is
possible to chose the functions $\hat\phi$ and $\hat\psi$ as
\begin{equation}\label{eq:phipsibethe}
  \hat\phi_i(x_i)  \egaldef \hat b_i(x_i), \qquad
  \hat\psi_a(\x_a) \egaldef\frac{\hat b_a(\x_a)}{\prod_{i\in a}\hat
b_i(x_i)},
\end{equation}
in order to obtain a prescribed set of beliefs $\hat b$ at a
fixed point. Indeed, BP will admit a fixed point
with $b_a=\hat b_a$ and $b_i=\hat b_i$ when $m_{a\to i}(x_i)\equiv 1$.
Since only the beliefs matter here, without loss of generality,
we restrict ourselves in the remainder of this section to the
functions~\eqref{eq:phipsibethe}. Then, from
\eqref{eq:jacobn}, the definition of $\widetilde J$ rewrites
\begin{equation*}
\widetilde J_{ai,k}^{a'j,\ell}
\egaldef\biggl[b_{k\ell}^{(iaj)}- \frac{1}{q}\sum_{x=1}^q b_{x\ell}^{(iaj)}
\biggr] A_{ai}^{a'j}
= J_{ai,k}^{a'j,\ell} - \frac{1}{q}\sum_{x=1}^q J_{ai,x}^{a'j,\ell}.
\end{equation*}

For each connected pair $(i,j)$ of variable nodes, we associate to the
stochastic kernel $B^{(iaj)}$ a combined stochastic kernel
$K^{(iaj)}\egaldef B^{(iaj)}B^{(jai)}$.
In the following we consider $b_i$ as a vector of $\mathbb{R}^q$. Since
$b_i\,B^{(iaj)}=b_j$, $b_i$ is the invariant measure associated to $K$:
\[
 b_i\,K^{(iaj)}=b_i\,B^{(iaj)}B^{(jai)} 
 = b_j\,B^{(jai)} = b_i,
\]
and $K^{(iaj)}$ is reversible, since
\begin{equation*}
b_i(k) K_{k\ell}^{(iaj)} = \sum_{m=1}^q b_{mk}^{(jai)} 
b_j(m) b_{m\ell}^{(jai)} = \sum_{m=1}^q b_{mk}^{(jai)} b_{\ell m}^{(iaj)}
b_i(\ell) = b_i(\ell)\,K_{\ell k}^{(iaj)}.
\end{equation*}
Let $\mu_2^{(iaj)}$ be the second largest eigenvalue of $K^{(iaj)}$ and
let
\begin{equation*}
\mu_2 \egaldef \max_{(iaj)} \sqrt{|\mu_2^{(iaj)}|}.
\end{equation*}
The combined effect of the graph and of the local correlations on the stability of
the reference fixed point is stated as follows.
\begin{theo}\label{thm:stability}
Let $\lambda_1$ be the Perron eigenvalue of the matrix $A$
\begin{enumerate}
\item[(i)] if $\lambda_1\mu_2 < 1$, the fixed point of BP scheme
(\ref{eq:normrule}, \ref{eq:Zmess}) associated to $b$ is stable.
\item[(ii)] If the system is homogeneous ($B^{(iaj)} = B$ independent of $i$, $j$
and $a$), $\lambda_1 \mu_2 \leq 1$ is also a necessary condition.
\end{enumerate}
\end{theo}

Condition (i) combines the effects of a term ($\mu_2$) which depends on the local
dependence structure of the given fixed point with another one ($\lambda_1$)
characteristic of the underlying graph. For example, in the homogeneous case, if $\G$
has uniform degrees $d_a$ and $d_i$, the condition reads
\[
\mu_2 (d_a-1)(d_i-1) < 1.
\]
In the case of binary variables $\mu_2^{(iaj)} = \det(K^{(iaj)})$, which is just the
square of Pearson's correlation coefficient between $x_i$ and $x_j$, which in
general depends on the factor $a$. The condition (i) of Theorem \ref{thm:stability}
thus is an upper bound on the correlations between variables at stable fixed points.

In order to prove part (i) of the theorem, we will consider a local norm on
$\mathbb{R}^q$ attached to each variable node $i$,
\[
\|x\|_{b_i} \egaldef \Bigl(\sum_{k=1}^q
x_k^2 b_i(k)\Bigr)^{\frac{1}{2}}\ 
\text{and}\ 
\langle x\rangle_{b_i} \egaldef \sum_{k=1}^q x_k b_i(k),
\]
the local average of $x\in\mathbb{R}^q$ w.r.t $b_i$. For
convenience, we will also consider the somewhat hybrid global norm on
$\mathbb{R}^{q\times|\E|}$
\[
\|x\|_{\boldsymbol\pi,b} \egaldef \sum_{(ai) \in \E}\pi_{ai}\|x_{ai}\|_{b_i},
\]
where $\boldsymbol\pi$ is the right Perron vector of $A$, associated to
$\lambda_1$. We have the following useful inequality:
\begin{lem}\label{lem:ineq}
For any $(x^{(i)},x^{(j)})\in \mathbb{R}^q\times\mathbb{R}^q$, such that $\langle
x^{(i)}\rangle_{b_i} = 0$ and $x^{(j)}_\ell\,b_j(\ell) = \sum_k x^{(i)}_k\,
b_i(k) B_{k\ell}^{(iaj)}$,
\[
\langle x^{(j)}\rangle_{b_j} = 0
\qquad \text{and}\qquad 
\|x^{(j)}\|_{b_j}^2 \le \mu_2^{(iaj)} \|x^{(i)}\|_{b_i}^2.
\]
\end{lem}
\begin{proof}
By definition of the kernels $K^{(iaj)}$, we have
\begin{equation*}
\|x^{(j)}\|_{b_j}^2 = \sum_{k=1}^q \frac{1}{b_j(k)} \Bigl|\sum_{\ell=1}^q b_{\ell
k}^{(iaj)}b_i(\ell)\,x_\ell^{(i)}\Bigr|^2 = \sum_{\ell,m} x_\ell^{(i)}x_m^{(i)}
K_{\ell m}^{(iaj)}b_i(\ell).
\end{equation*}
Since $K^{(iaj)}$ is reversible, Rayleigh's theorem implies
\[
\mu_2^{(iaj)} \egaldef \sup_{x}
\Bigl\{\frac{\sum_{k\ell}x_k x_\ell K_{k\ell}^{(iaj)}b_i(k)}{\sum_k
x_k^2 b_i(k)},
\langle x\rangle_{b_i} = 0,x\ne 0\Bigr\},
\]
which concludes the proof.
\end{proof}
To deal with iterations of $J$, we express it as a sum over paths.
\[
\bigl(J^n\bigr)_{ai,k}^{a'j,\ell} = \bigr(A^n\bigr)_{ai}^{a'j} 
\bigl(B_{ai,a'j}\n\bigr)_{k\ell}\ ,
\]
where $B_{ai,a'j}\n$ is an average stochastic kernel,
\begin{equation}\label{eq:paths}
B_{ai,a'j}\n \egaldef \frac{1}{|\Gamma_{ai,a'j}\n|}
\sum_{\gamma\in\Gamma_{ai,a'j}\n}\prod_{(ck,d\ell)\in\gamma}B^{(kc\ell)}. 
\end{equation}
$\Gamma_{ai,a'j}\n$ represents the set of directed path of length
$n$ joining $ai$ and $a'j$ on $L(\G)$ and its cardinal is precisely
$|\Gamma_{ai,a'j}\n| = \bigr(A^n\bigr)_{ai}^{a'j}$.
\begin{lem}\label{lem:ineq2}
For any $(x^{(ai)},x^{(a'j)})\in \mathbb{R}^{2q}$, such that $\langle
x^{(ai)}\rangle_{b_i} = 0$ and
\[
 x^{(a'j)}_\ell\,b_j(\ell) = \sum_k
x^{(ai)}_k\,b_i(k)\bigl(B_{ai,a'j}\n\bigr)_{k\ell}\ ,
\] 
the following inequality holds
\[
\|x^{(a'j)}\|_{b_j} \le \mu_2^n \|x^{(ai)}\|_{b_i}.
\]
\end{lem}
\begin{proof}
Let $x^{(a'j)}(\gamma)$ be the contribution to $x^{(a'j)}$ corresponding to the
path $\gamma\in\Gamma_{ai,a'j}\n$. Using Lemma~\ref{lem:ineq}
recursively yields for each individual path
\[
\|x^{(a'j)}(\gamma)\|_{b_j} \le \mu_2^n \|x^{(ai)}\|_{b_i},
\]
and, owing to triangle inequality,
\[
\|x^{(a'j)}\|_{b_j} \le
\frac{1}{|\Gamma_{ai,a'j}\n|}\sum_{\gamma\in\Gamma_{ai,a'j}\n}
\|x^{(a'j)}(\gamma)\|_{b_j} \le \mu_2^n \|x^{(ai)}\|_{b_i}.
\]
\end{proof}
\begin{proof}[Proof of Theorem~\ref{thm:stability}]
Let $\v$ and $\v'$ two vectors with $\v' = \v \tilde J^n =
\v(\I-M)J^n$, since $\tilde J M = 0$. Recall that the
effect of $(\I-M)$ is to first project on a vector with zero local
sum, $\sum_k\bigl(\v(\I-M)\bigr)_{ai,k} = 0,\ \forall i\in\V$, so we
assume directly $\v$ of the form
\[
v_{ai,k} = x_{ai,k}\,b_i(k),\qquad\text{with}\qquad 
\langle x_{ai}\rangle_{b_i} = 0.
\]
As a result, $\v' = \v J^n$ is of the same form. Let
$x'_{a'j,\ell}\egaldef v'_{a'j,\ell}/b_j(\ell)$. We have
\[
\|x'\|_{\pi,b} 
\le \sum_{(a'j)\in\E}\pi_{a'j}\sum_{(ai)\in\E}\bigl(A^n\bigr)_{ai}^{a'j}
\|y^{(ai)}_{a'j}\|_{b_j},
\]
with $y^{(ai)}_{a'j,\ell}\ b_j(\ell) = \sum_k
x_{ai,k}\,b_i(k)\bigl(B_{ai,a'j}\n\bigr)_{k\ell}$.
Applying Lemma~\ref{lem:ineq2} to $y^{(ai)}_{a'j}$ yields
\begin{align*}
\|x'\|_{\pi,b} 
&\le
\sum_{(a'j)\in\E}\pi_{a'j}\sum_{(ai)\in\E}\bigl(A^n\bigr)_{ai}^{a'j}\mu_2^n\|x_{ai}
\|_ { b_i }
= \lambda_1^n\mu_2^n \|x\|_{\pi,b},
\end{align*}
since $\boldsymbol\pi$ is the right Perron vector of $A$. This ends the proof of (i).

For (ii), when the system is homogeneous, $\widetilde J$ is a tensor product of $A$
with
$\widetilde B$, and its spectrum is therefore the product of their
respective spectra. 
\end{proof} 

The quantity $\mu_2$ is representative of the level of mutual
information between variables. It relates to the spectral gap (see
e.g.~\cite{DiaStr} for geometric bounds) of each elementary stochastic
matrix $B^{(iaj)}$, while $\lambda_1$ encodes the statistical
properties of the graph connectivity. The bound $\lambda_1\mu_2 <1$
could be refined when dealing with the statistical average of the sum
over path in (\ref{eq:paths}) which allows to define $\mu_2$ as
\[
\mu_2 = \lim_{n\to\infty}\max_{(ai,a'j)}
\Bigl\{\frac{1}{|\Gamma_{ai,a'j}\n|}
\sum_{\gamma\in\Gamma_{ai,a'j}\n}
\Bigl(\prod_{(x,y)\in\gamma}\mu_2^{(xy)}\Bigr)^{\frac{1}{2n}}\Bigr\}.
\]

\subsection{Local convergence in quotient space $\N \setminus \W$}
\label{ssec:b2qs}
We make here the connexion with the notion of local stability in the
quotient space $\N \setminus \W$ of Section~\ref{sec:bdynamic}.
Trivial computations yield $\nabla \Lambda = J$.
In terms of convergence in $\N\setminus\W$, the stability of a fixed point
is given by the projection of $J$ on the quotient space $\N \setminus \W$ 
and we have \cite{MooijKappen07}:
\[
[J] \egaldef [\nabla \Lambda] = \nabla [\Lambda]
\]
The normalization $Z^\mathrm{mess}_{ai}$ is in fact just a way to compute $[J]$ 
by applying a projection $\I-M$ to $J$. Since $\text{ker}(\I-M) = \W$,
it is just a quotient map from $\N$ to $\N\setminus \W$. For any
differentiable positively homogeneous normalization, we obtain the same result, the
Jacobian of the corresponding normalized scheme is the projection of $J$ on
$\N\setminus\W$, through some quotient map.

\section{Conclusion}
\label{sec:conclusion}
We provided here, for the first time at our knowledge, an explicit
sufficient condition for local stability of a belief propagation fixed
point, instead of sufficient conditions for convergence to a unique
fixed point. This condition is coherent with the usual understanding
of BP convergence; when the connectivity of both $\G$ and $L(\G)$
increases, $\lambda_1$ is also increasing since $A$ is increasing. So
Theorem $\ref{thm:stability}$ imposes that the level of mutual
information $\mu_2^{(iaj)}$ between variables $i$ and $j$ at a stable
fixed point decreases. Reciprocally, the sparser $\G$ is, the bigger
mutual information can be. This somewhat explains why BP performs
better on sparse graphs: the amount of admissible mutual information
between variables at a stable fixed point is larger on a sparse graph
than on a dense one.

\bibliography{refer}
\bibliographystyle{plain}

\end{document}